\documentclass{article}

\PassOptionsToPackage{square,numbers}{natbib}
\usepackage[preprint]{neurips_2026}

\usepackage[utf8]{inputenc} 
\usepackage[T1]{fontenc}    
\usepackage{hyperref}       
\usepackage{url}            
\usepackage{booktabs}       
\usepackage{amsfonts}       
\usepackage{nicefrac}       
\usepackage{microtype}      
\usepackage{xcolor}         

\newcommand{\system}{PRISMat}
\usepackage{amsmath}
\usepackage{algorithm}
\usepackage{algpseudocode}
\usepackage{wrapfig}
\DeclareMathOperator{\LatticeGenerator}{LatticeGenerator}
\DeclareMathOperator{\AtomGenerator}{AtomGenerator}
\DeclareMathOperator{\PositionGenerator}{PositionGenerator}
\DeclareMathOperator{\uniform}{uniform}
\usepackage{cleveref}
\usepackage{graphicx}
\usepackage{subcaption}
\newtheorem{proposition}{Proposition}
\newtheorem{proof}{Proof}
\title{PRISMat: Policy-Driven, Permutation-Invariant Autoregressive Material Generation}

\author{%
  Claire Schlesinger \\
  Khoury College of Computer Sciences \\
  Northeastern University \\
  Boston, MA 02115 \\
  \texttt{schlesinger.e@northeastern.edu} \\
  \And
  Circe Hsu \\
  Khoury College of Computer Sciences \\
  Northeastern University \\
  Boston, MA 02115 \\
  \texttt{hsu.circe@northeastern.edu} \\
  \AND
  Peter Schindler \\
  College of Engineering \\
  Northeastern University \\
  Boston, MA 02115 \\
  \texttt{p.schindler@northeastern.edu} \\
  \And
  Robin Walters \\
  Khoury College of Computer Sciences \\ 
  Northeastern University \\
  Boston, MA 02115 \\
  \texttt{r.walters@northeastern.edu} \\
}

\begin{document}

\maketitle

\begin{abstract}
Rapid identification of candidate materials with target properties has become a key task in materials science. Machine learning has emerged as an alternative to physics-based simulation, offering a faster and cheaper way to filter materials based on their stability and other target properties, reducing the number of candidates that reach the costly synthesis stage. Recently, Large Language Models (LLMs) have been applied to this role, but these models are parameter-heavy and computationally expensive both during training and at inference time, making them unsuitable for high-throughput tasks. This inefficiency stems from both the large over-parameterization of language models and the difficulty of framing material generation as a sequence learning problem. In this paper, we present \system{}, a cost-effective, permutation-invariant model, which addresses these limitations. We show that \system{}, despite taking less time for inference, is able to outperform LLMs in generating crystal slabs conditioned on critical materials' surface properties. In targeted material discovery, we achieve mean absolute errors of 0.188 eV/\AA$^2$ and 2.79 eV for cleavage energy and work function tasks, respectively, reducing the error of the next best model by 4$\times$.
\end{abstract}

\section{Introduction}

\begin{wrapfigure}{R}{0.55\textwidth}
    \centering
    \includegraphics[width=0.55\textwidth]{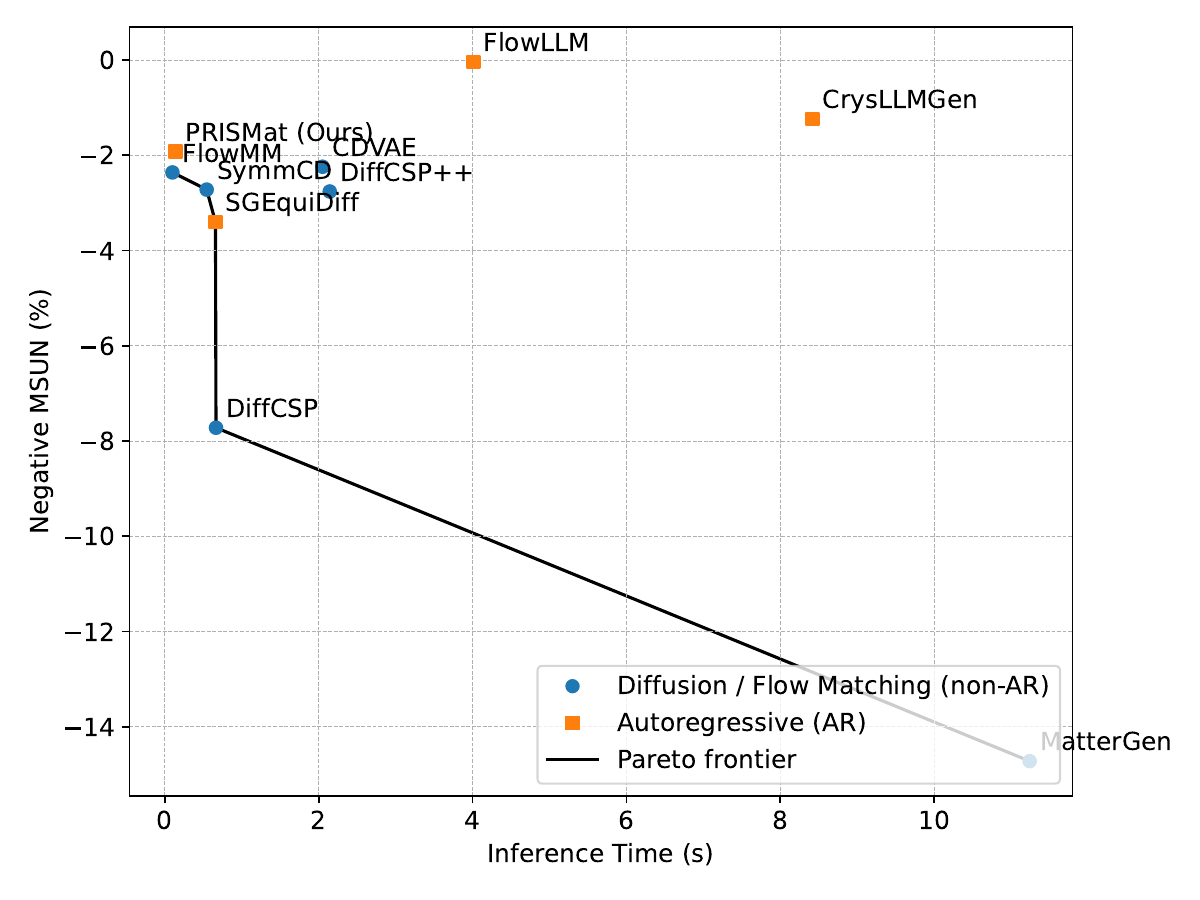}
    \caption{Pareto frontier of time it takes to generate a single crystal versus the negative rate of metastable, unique, and novel (MSUN) crystals. Orange squares indicate autoregressive techniques, while blue circles indicate pure diffusion techniques. \system{} performs on par with the best diffusion models and clearly outperforms other autoregressive LLM systems, showing that it is the most efficient autoregressive system.}
    \label{fig:pareto}
\end{wrapfigure}

Identifying novel materials is slow and expensive due to the vast search space and the difficulty of synthesizing and evaluating candidates. The challenge is even greater when trying to identify candidates with specific desirable properties. To accelerate this process, machine learning systems have been used to both rapidly evaluate the stability and properties of candidates and to discover new materials \cite{Riebesell2025matbenchdiscovery, rhodes2025orbv3atomisticsimulationscale, wood2025umafamilyuniversalmodels, chibani2020mlapproachforpredictionofmaterialsproperties, chan2022applicationofmlforadvancedmaterialpredictionanddesign, kadulkar2022mlassisteddesign}. For example, several methods have been proposed to generate candidate materials with a target band gap or space group \cite{MatterGen2025, bone2025discoveryrecoverycrystallinematerials, chang2025sgequidiff, jiao2024diffcsppp}. However, these methods generate only bulk crystals, meaning idealized, defect-free structures that repeat infinitely in all three dimensions. In reality, all crystals are finite and therefore bounded by surfaces, which for many technologically relevant applications are the primary determinants of material properties. These properties are crucial for electron emission devices and heterogeneous catalysis \cite{Lin2023workfunction, belluci2026thermionic, chen2024workfunctionguidedelectrocatalystdesign, zeradjanin2017balancedworkfunction,radinger2022workfunction}. They also determine contact barriers at semiconductor interfaces and yield the approximate shape of nanoparticles \cite{Freeouf1981schottkybarrier, cankaya2004schottkybarrier, wulff1901, maxson2025metalsupportinteractions}.

Beyond their inability to capture real-world structural complexity, existing generative methods carry a further practical limitation: computational cost, with LLM-based approaches in particular requiring prohibitively long inference times for high-throughput screening. Typically, LLM-based methods generate materials in one shot, often directly outputting a Crystallographic Information File (CIF) or other crystal encoding to represent the structure. This method suffers from heavy redundancy: by reordering, many CIF files encode the same crystal. This dramatically enlarges the output space, making the distribution of ideal structures more difficult to learn. Current methods attempt to address this limitation by imposing a canonical order in the CIF or using extensive data augmentation during training \cite{chang2025sgequidiff, gruver2023CrystalTextLLM}. Enforcing an ordering to atoms through canonicalization does not consider any underlying physics and can lead to instabilities \cite{ma2024canonicalizationperspectiveinvariantequivariant}. Similarly, permuting atom order as an augmentation may encourage the model to learn permutation invariance, but it significantly increases training time without fully resolving the problem. Despite these drawbacks, autoregressive systems are still an attractive choice for policy-guided generation, as they can correct or reject certain crystals that do not fall within desired generation guidelines.

We propose \system{} (PeRmutation-Invariant Sequential Material generation), a system which combines autoregressive generation with flow matching to generate novel materials. This allows us to take advantage of the controllability of policy-driven autoregressive generation with the speedups provided by flow matching. \system{} addresses the issues with autoregressive generation by representing crystals in a more efficient format and by being trained in a permutation-invariant manner. \system{} is composed of three parts: 1)~A Gaussian mixture model which predicts lattice parameters, 2)~An $E(3)$ equivariant graph neural network (GNN) which autoregressively predicts atoms, and 3)~An $E(3)$ equivariant flow matching model to assign positions to each atom \cite{lipman2023flow, satorras2021egnn}. This three-part setup enables interventions after each step, so we can customize the generation process.

\system{} achieves permutation invariance by reinterpreting the conditional output distribution of the autoregressive model as the cumulative distribution of all atom types remaining in the crystal instead of the probability distribution of the immediate next token.  While next token prediction is well-suited for text generation, where order is critical, it is actually a hindrance in material generation, where no physically meaningful order exists.  By reinterpreting the output distribution, our method allows the use of the same underlying architectures used for text generation while enforcing permutation invariance appropriate to the setting.

Our contributions are:
\begin{itemize}
    \item We introduce \system{}, the first autoregressive, permutation invariant framework for crystal generation.
    \item We analyze \system{}'s performance to identify the benefit of permutation invariance and policy-guided generation as well as optimizing the sampling parameters for \system{}.
    \item We evaluate our method's ability to do conditional material generation on cleavage energy and work functions using the crystal slab dataset from \citet{schindler2024discoveryofstablesurfaces} and show it outperforms other autoregressive, conditioned LLM techniques by quartering the overall error on conditioned generation and having the lowest time per generated metastable, unique, and novel (MSUN) structure of any autoregressive model.
\end{itemize}

\section{Related Work}

\paragraph{Autoregressive Graph Generation}

Permutation invariance has long been a challenge in autoregressive graph generation. G-SchNet \cite{GSchnet} is an autoregressive method for the generation of rotationally invariant point graphs used for molecule generation. G-SchNet handles permutation invariance by relying on the structure of molecules by selecting an atom to focus on and then predicting an atom to bond to that atom. GraphRNN builds a graph by representing it as a unique sequence and predicting those sequences to generate graphs \cite{GraphRnn}. GraphRNN uses a BFS node ordering scheme to reduce the complexity due to the large number of possible node orderings. GCPN autoregressively constructs a graph using a generative adversarial network (GAN) and reinforcement learning policies as guidance during training \cite{you2018gcpn}. GCPN gets its permutation invariance through its discriminator by only looking at the final generated structure, and its reward function does not rely on the order of atom placement. Our method is distinct because, unlike GCPN, GSchNet, or GraphRNN, it works on crystals whose periodic structure means there is no canonical and physically meaningful ordering to predict atoms in.

\paragraph{Material Generation Via Diffusion}

Diffusion models are a popular choice for \textit{de novo} crystal generation \cite{xie2021cdvae, jiao2023diffcsp, levy2024symmcd, jiao2024diffcsppp, miller2024flowmm, MatterGen2025, chang2025sgequidiff}. CDVAE uses a variational autoencoder (VAE) to add the ability to do inverse design, the process of creating a material with a desired property, and then uses a diffusion model to denoise the output to produce the new material \cite{xie2021cdvae}. Our method differs because, rather than using a VAE and optimizing for the property in latent space, we use autoregressive generation and directly condition on the desired properties, simplifying training and generation. Another approach to \textit{de novo} crystal generation is DiffCSP, a pure diffusion model which utilizes a periodic $E(3)$ invariant GNN \cite{jiao2023diffcsp}. An alternative to diffusion-based methods is flow matching, with FlowMM leveraging Riemannian flow matching to simultaneously predict crystal structures and atom types \cite{miller2024flowmm}. We also use Riemannian flow matching to predict the atom positions, but use autoregressive generation for atom types, as it allows for fewer generation steps. 

In crystals, symmetry is fully described by space groups, which classify all symmetry operations consistent with three-dimensional lattice periodicity. These space groups are highly relevant to chemical properties and the structure of the unit cell. SymmCD decomposes the unit cell of a crystal into the asymmetric unit, the smallest unit that can reproduce the unit cell through these symmetry transformations, and predicts it with a diffusion model \cite{levy2024symmcd}. MatterGen, DiffCSP++, and SGEquiDiff are all diffusion models that use information about the space group when generating the crystal unit cell \cite{MatterGen2025, jiao2024diffcsppp, chang2025sgequidiff}. Mattergen is conditioned on the space group number. DiffCSP++ uses the restrictions on atom counts and positions to help with diffusion, and SGEquiDiff is fully equivariant to the space group, using the restrictions on atoms and Wyckoff positions to place and diffuse the structure. Interestingly, SGEquiDiff uses an autoregressive method to predict atom types and Wyckoff positions before diffusing their position, but enforces an ordering on the types of atoms predicted. Our method is similar to SGEquiDiff, but it does not require a predefined atom ordering. Instead, permutation invariance is enforced during training through the choice of loss function. In addition, our approach does not incorporate any explicit information about the crystal’s space group, which increases our speed in inference.

\paragraph{Autoregressive Material Generation}

Some autoregressive methods use LLMs to predict novel materials by predicting a Crystallographic Information File (CIF) \cite{antunes2024CrystaLLM, gruver2023CrystalTextLLM, sriram2024flowllm, khastagir2025CrysLLMGen}. CrystaLLM trains an LLM from scratch on CIFs to predict novel CIFs \cite{antunes2024CrystaLLM}. CrystaLLM-$\pi$ expands on CrystaLLM by adding in a system to pass property values directly into every layer of the transformer rather than in the text prompt \cite{bone2025discoveryrecoverycrystallinematerials}. CrystalLLM uses a pretrained LLM and finetunes it on additional CIFs in order to produce novel CIFs \cite{gruver2023CrystalTextLLM}. FlowLLM takes the CIFs output from CrystalLLM and uses a flow matching model to refine the outputs \cite{sriram2024flowllm}. CrysLLMGen is similar to FlowLLM but uses a diffusion model for refinement \cite{khastagir2025CrysLLMGen}. LLMatDesign differs from others by using an LLM starting from an initial composition and design conditions and using it to autoregressively predict changes until the desired material properties are achieved \cite{jia2024llmatdesignautonomousmaterialsdiscovery}. While our method is inspired by LLM-style training and generation, it does not operate on language, as the CIF is a poorer representation of crystals than the unit cell. Instead, it directly predicts atom types and does not employ causal masking, opting instead to enforce permutation invariance in the model design.

\section{Background}

In this section, we cover the necessary background information to understand \system{}. We overview a mathematical formulation of crystal unit cells, outline the process of autoregressive generation and Riemannian flow matching, provide a formal definition of equivariance, and cover the necessary background information relevant to our conditional slab generation task.

\subsection{Definition of a Crystal Cell and Crystal Slab}

The defining characteristic of a crystal is the periodic structure of atoms. Due to this periodicity, crystals can be compactly represented by a unit cell, a parallelepiped-shaped subsection of the crystal. The unit cell provides a computationally efficient representation of the infinitely repeating crystal structure. The full crystal can be reconstructed by tiling the unit cell along all three lattice basis vectors. A crystal $C$ is defined by a tuple $C = (L, A, X)$ where $L = (l_1, l_2, l_3) \in \mathbb{R}^{3\times3}$ are the three lattice vectors that define the periodic boundaries of the crystal, $A = (a_1, a_2, \ldots, a_N) \in atoms^N$ are the atom types, where $atoms$ is a set of available elements, $N$ is the number of atoms in the unit cell, $X = (x_1, x_2, \ldots, x_N) \in [0,1)^{N \times 3}$ are the fractional coordinates in the crystal which show the atoms' positions in the unit cell as a fraction of the distance along the lattice vectors. To get the Cartesian coordinates from the fractional coordinates and lattice vectors, we simply multiply the fractional coordinates by the lattice vectors $XL^T=X_{\mathrm{Cartesian}}$.

A crystal slab can still be defined by a crystal tuple $C = (L, A, X)$, but slab structures inherit a broken $E(3) \rightarrow SO(2)$ symmetry via cutting of the bulk crystal structure, creating a unique design challenge for fully equivariant models. Crystal slabs are more reflective of realistic structures as they contain the termination that occurs at the boundary of a crystal. Here, we consider two critical properties of crystal slabs: the cleavage energy and work function. The cleavage energy is the amount of energy required to split the crystal along a specific Miller index. The cleavage energy (which is equal to the surface energy for symmetric slabs) determines the stability of surfaces. A slab has two work functions, one for the bottom surface of the slab and one for the top. The work function is the amount of energy necessary to free an electron from the surface of a slab.

\subsection{Autoregressive Generation}

Autoregressive generation is the process of predicting a sequence using the previously predicted elements of the sequence. A model $p$ works by producing a distribution over the tokens in the vocabulary $V$, $p(v_t|v_0,v_1, \ldots, v_{t-1})$ where $v_i \in V$, and selecting the next token from that distribution.

The variance of the distribution $p$ is controllable by changing the temperature, $\tau$, and nucleus sampling, $P$, hyperparameters. The temperature modulates the softmax function. Temperatures that are less than $1$ increase the likelihood of more probable tokens, while temperatures that are greater than $1$ increase the likelihood of less probable tokens from the original distribution. Nucleus sampling cuts off very improbable tokens, which increases consistency.

\subsection{Equivariance}

A function $f\colon X \rightarrow Y$ is equivariant to a group $G$ if for any $x \in X$ and any $g \in G$, $f(gx) = gf(x)$. The function $f$ is invariant if $f(gx) = f(x)$. In this work, we use $E(3)$ equivariant GNNs where $E(3)$ is the group of all rotations, reflections, and translations over $\mathbb{R}^3$ \cite{satorras2021egnn}. 
$E(3)$ equivariance is a desired property when working with crystals, as the crystal's properties are unchanged by any rotation or translation. In crystal generation, the problem is $E(3)$ invariant, as no matter the orientation, chirality, or position of the crystal, the target unit cell remains constant.

\subsection{Riemannian Flow Matching}

Flow matching is the process of learning a vector field from some probability distribution, usually a normal distribution, to the true data distribution \cite{lipman2023flow}. A flow map is learned from data to noise, while generation reverses the process to go from noise to data. Velocities $u_t$ are predicted, going from the true data distribution to the noise distribution. The velocities are computed at train time by taking a random sample from the random initial data distribution and the true data distribution and finding the displacement between them.

Due to the periodicity inherent to crystal structures, it is necessary to define a manifold properly capturing this periodicity, on which the flow is learned. Without this periodicity, it is possible to predict positions that fall outside the unit cell, requiring a postprocessing step to map external atoms back into the unit cell. Riemannian flow matching extends the conventional flow matching algorithm to enable flows on general geometries \cite{chen2024flow}. This allows us to implement flows that respect the periodicity of the unit cell. To remove the redundant positions from our case, we use a 3-torus manifold as it correctly implements the periodicity of the unit cell. 

\section{Method}

\begin{figure}[t]
    \centering
    \includegraphics[width=\linewidth]{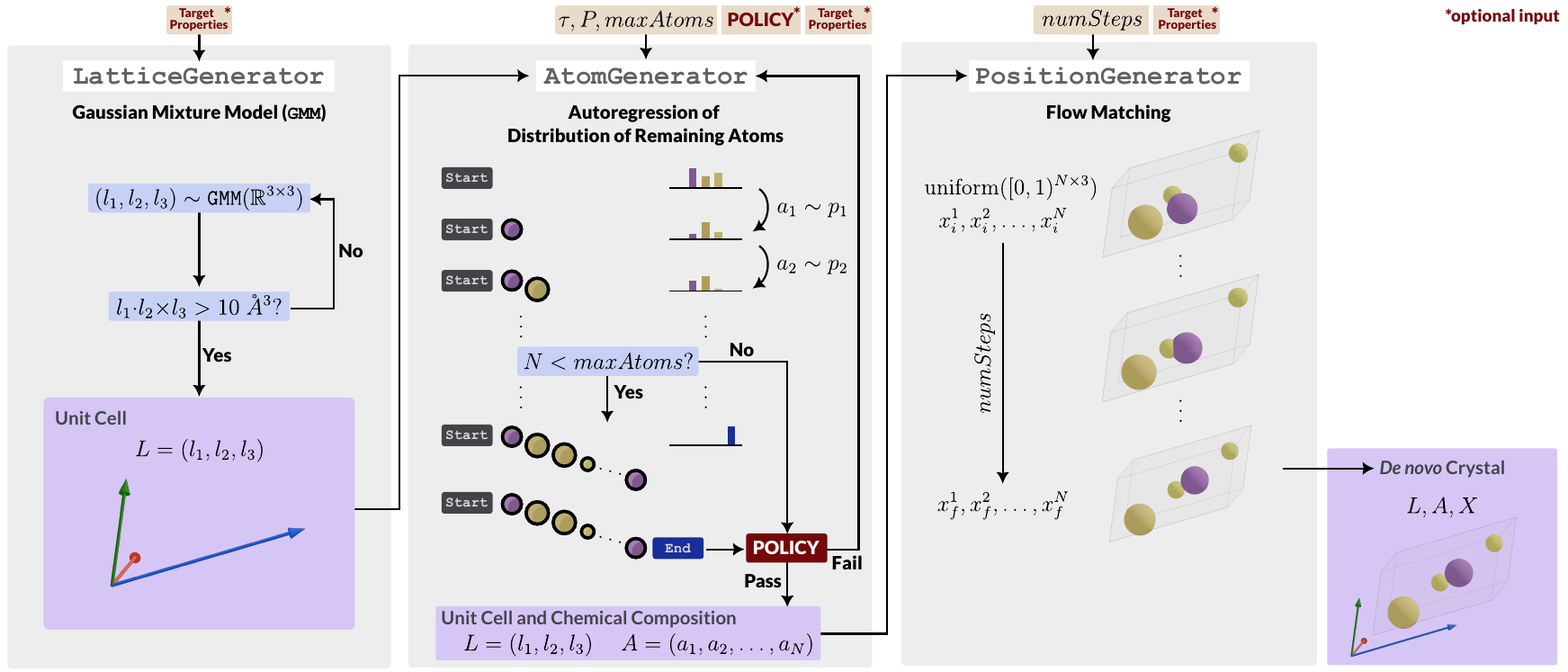}
    \caption{The three stages of \system{}. We start by predicting the periodic boundaries of the unit cell. We then autoregressively predict each atom until we reach an <end> token. Finally, we use flow matching to position each atom in the unit cell. By splitting generation into three stages, we allow for more control over generation and faster rejection of impossible crystals. Mandatory user inputs are $\tau$, $P$, $maxAtoms$, and $numSteps$, while the policy and target materials properties for conditional generation are optional. The pseudocode can be found in \cref{alg:automatgen-sample}.}
    \label{fig:system-diagram}
\end{figure}

\system{} is made up of three parts, each responsible for building a section of the crystal: the $\LatticeGenerator$, a Gaussian mixture model which is responsible for generating the lattice vectors, the $\AtomGenerator$, an $E(3)$ invariant GNN which autoregressively predicts the atoms in a crystal, and the $\PositionGenerator$, an $E(3)$ equivariant GNN which runs a flow matching algorithm to predict the crystal structure. In this section, we discuss each portion of \system{}, how they operate, and why that method was selected. A full overview of \system{} can be seen in \cref{fig:system-diagram}. \Cref{fig:generated-examples} showcases a real crystal from MP-20 compared to a crystal \system{} generated after being trained on MP-20.

\subsection{$\LatticeGenerator$}

\system{} starts by predicting the three lattice vectors that make up the periodic boundary of the material. We build the $\LatticeGenerator$ by training a Gaussian mixture model using expectation maximization over the lattice parameters from our training dataset. $\LatticeGenerator$ is therefore a distribution that we can sample from to get a set of lattice parameters $lattice \sim \LatticeGenerator$. Since periodic boundaries are mutually dependent, we model them jointly using a Gaussian mixture over $\mathbb{R}^{3 \times 3}$. The periodic boundaries also determine the size of the unit cell, so we reject any periodic boundaries that produce unit cells of size less than $10 \mathrm{\AA}^3$. 

Periodic boundaries partially determine the bonding angles of the atoms on the edge of the unit cell, so by predicting the periodic boundaries first, we provide some chemical information about the atoms that should belong in the unit cell. Because of this, we provide the lattice information to all subsequent parts of \system{} as it will help determine atom types and their positions. Additionally, it also enables more design choices by allowing us to restrict boundaries to specific crystal systems.

\subsection{$\AtomGenerator$}

After predicting the lattice vectors, we autoregressively predict the atom types. We define two virtual nodes besides the standard atom nodes, a <start> node and an <end> node. The <start> node is used to prompt the model to generate additional nodes, while the <end> node is used to stop generation and implicitly determine the number of atoms in the unit cell. 

After the start node, we autoregressively generate a sequence of atoms types until we generate an <end> token. As input, $\AtomGenerator$ takes the lattice parameters sampled from $\LatticeGenerator$ and a <start> token. If the atoms $a_0,\ldots,a_{t-1}$ have already been generated, then the model $\AtomGenerator$ predicts the distribution of next possible atom types $p_t = \AtomGenerator(lattice, start, a_0, a_1, \ldots, a_{t-1})$. We sample from this distribution to get the next atom type $a_t \sim p_t$. We then use this predicted atom in the next step of atom generation. This repeats until $a_t = END$. We use an E(3) invariant GNN, as neither the orientation of the input lattice vectors nor the order of the input atoms should affect the final output \cite{satorras2021egnn}. The input atoms are represented as a fully connected graph where the lattice parameters are added into the node features.

\paragraph{Permutation Invariant Autoregression.}

\begin{figure}[t]
    \centering
    \includegraphics[width=\textwidth]{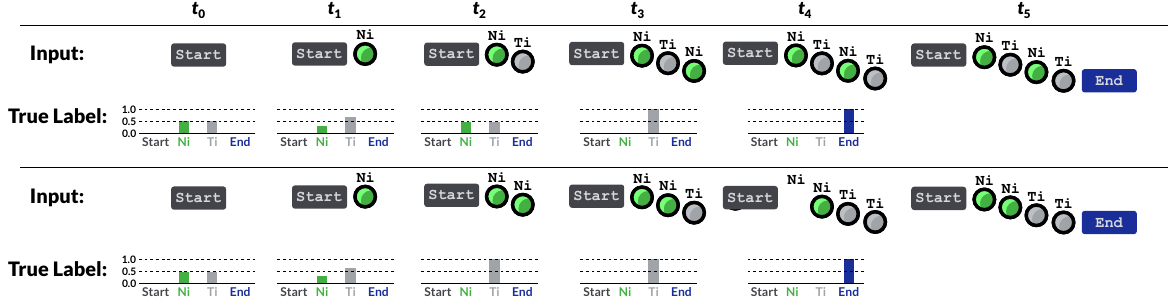}
    \caption{A visualization of the training input and labels provided for a fictitious nickel-titanium crystal. The true distribution has been shortened to only show <start>, nickel, titanium, and <end> in that order. Despite the order in which the atoms are added to the training input, steps $t_0$, $t_1$, and $t_3$ have the same true label as they have the same input atoms, whereas there are two possible orderings shown for step $t_2$. By design, the final chemical composition is identical for both permutations.}
    \label{fig:training-diagram}
\end{figure}

To train the autoregressive atom predictor, rather than imposing an implicit ordering on the atoms it predicts, we instead predict a categorical distribution over all atoms that remain in the unit cell at each time step. If no atoms remain in the unit cell, we predict a distribution that is 100\% the <end> token. Since this distribution remains the same no matter how the remaining atoms in the crystal are ordered, training becomes permutation invariant. In training, we minimize the KL-divergence loss between the output distribution of $\AtomGenerator$ and the true categorical distribution of the remaining atoms \cite{kldivergence}. Generally, if we are learning to predict a crystal $C = (L,A,X)$, then the model $\AtomGenerator$ would take some subset of the atoms $a \subseteq A$ and learn to predict the categorical distribution of the remaining atoms $\Bar{a} = A \setminus a$. Compared to the training algorithms of LLMs or SGEquiDiff where the training target is a single token, our distributional approach makes it clear that there can be multiple feasible tokens at each training step. A proof showcasing our permutation invariant training can be found in \cref{prop:train-inv}. A full description of the training algorithm can be found in \cref{alg:automatgen-train}. To illustrate this process, \cref{fig:training-diagram} illustrates a fictitious Ni$_2$Ti$_2$ crystal and shows two possible permutations and the corresponding true labels (distributions) that ensure that the final outcomes are identical.

\paragraph{Policy Guidance} 

Part of the benefit of autoregressive generation is that fewer overall steps are needed to generate the atom types compared to diffusion models. This allows for faster rejection and resampling based on policy guidance. We try three policy methods: 1) \textit{partial policy} where we reject if any step of the partially generated result seems unrealistic during atom generation, 2) \textit{full policy} where we reject after the atom generation step if it seems like an unrealistic set of atoms and lattice parameters, and 3) \textit{SMACT} where we use the SMACT solver on the final chemical composition to reject if it is not charge-balanced, assuming known oxidation states of individual elements \cite{Davies2019}. 

To train the partial and full policy methods, we utilize an $E(3)$ invariant GNN and train it on the MP-20 dataset, where crystals in the MP-20 dataset are treated as realistic examples, and fake examples are made by removing, adding, or changing atoms in those realistic examples. The models are trained to classify whether a given crystal is real or fake based on the involved atoms and lattice parameters. Any crystals these policies mark fake are rejected during generation.

\subsection{$\PositionGenerator$}

The final step of our procedure is a Riemannian flow matching step to determine atom positions. We use an $E(3)$ equivariant GNN to construct $\PositionGenerator$ \cite{satorras2021egnn}. We instantiate our positions by $X_i \sim \uniform([0,1)^{N \times 3})$ and then run Riemannian flow matching using a 3-torus as our manifold. We utilize flow matching as it is the most efficient method of position generation, and because positioning is a global problem, global flow matching is most effective.

To train $\PositionGenerator$ for each crystal $C = (L,A,X)$ in our dataset, we generate a set of random positions $X' \sim \uniform([0, 1)^{N\time3}$ and compute the velocities $V = X' - X$. We then sample a random timestep $t \sim \uniform([0,1])$ and compute a position $X_{in} = (1-t)X + tX'$. Finally we train $\PositionGenerator$ to minimize $||\PositionGenerator(L, A, X_{in}, t) - V||^2_2$.

\subsection{Property-Targeted Generation}

\system{} has the ability to be conditioned to produce materials with desired properties. We change the Gaussian mixture model in the first step to a conditional one that changes means and standard deviations by applying Gaussian conditioning to each component and reweighting the mixture so the result is a new Gaussian mixture model representing the distribution of target variables given observed ones \cite{cgmm2024}. Additionally, for the GNNs in the atom and position generation steps, we concatenate the desired property values to the input features to allow property-targeted generation. 

\section{Experiments}

In this section, we examine the benefits of our method and show how permutation invariance, autoregressive generation, policy guidance, and flow matching enable a fast and efficient model that can generate novel materials. We evaluate our performance over the MP-20 subset of the materials project and Lemat-GenBench, as they are a standard crystal training dataset and evaluation set for \textit{de novo} generation \cite{Jain2013materialsproject, xie2021cdvae, betala2026lematgenbenchunifiedevaluationframework}. Additionally, we evaluate the performance of \system{} over conditional generation of crystal slabs using the dataset from \citet{schindler2024discoveryofstablesurfaces} to show that our model is able to handle larger, more realistic structures and conditional generation better than other autoregressive systems with conditional generation. We find that \system{} is able to be successfully policy driven while maintaining the necessary speed to be the fastest autoregressive model.

\begin{table}[t]
    \centering
    \caption{Comparison of \system{} versus an ablation using the same GNN setup but trained to learn permutation invariance rather than using the permutation invariant loss function. We see that \system{} outperforms the ablation in almost all categories.}
    \resizebox{\textwidth}{!}{
        \begin{tabular}{lccccccccc}
            \toprule
            &Valid & Unique & Novel & & & Metastable & MSUN \\
            Model  & (\%) $\uparrow$ & (\%) $\uparrow$ & (\%) $\uparrow$ & JSDistance $\downarrow$ & MMD $\downarrow$ & (\%) $\uparrow$ & (\%) $\uparrow$ \\
            \hline
            \system{} & \textbf{83.72} & \textbf{82.56} & \textbf{63.88} & \textbf{0.4682} & \textbf{0.00572} & 14.88 & \textbf{1.36} \\
             \ \ \ w/o Perm. Inv.  & 82.16 & 80.56 & 62.92 & 0.4845 & 0.00979 & \textbf{15.00} & 1.00 \\
            \bottomrule
        \end{tabular}
    }
    \label{tab:ablation}
\end{table}

\subsection{Ablation} 

To show that permutation invariance improves performance, we run an ablation of our method where we remove the permutation invariance and instead have the model learn permutation invariance by randomly shuffling the order of the input and output atoms at every training step. Under these conditions, we see a drop in performance from MSUN of 1.36\% to a MSUN of 1.00\%. Additionally, the ablated method also underperforms on all distributional and validity categories, as shown in \cref{tab:ablation}. This drop is significant as it means it would take much longer to find materials in high-throughput applications. This likely means that other methods, like SGEquiDiff, would also see an improvement in performance if they utilized a permutation invariant approach to atom generation.

\begin{table}[t]
    \centering
    \caption{Comparison of \system{} for various hyperparameters, trained on the MP-20 dataset. The best sampling parameters are $\tau=0.7$ and $P=0.9$ as it produces the highest MSUN\ rate.}
    \resizebox{\textwidth}{!}{
        \begin{tabular}{ccccccccccc}
           \toprule
           \multicolumn{2}{l}{Hyperparameters} & Valid & Unique & Novel & & & Metastable & MSUN \\
           $\tau$ & $P$ & (\%) $\uparrow$ & (\%) $\uparrow$ & (\%) $\uparrow$ & JSDistance $\downarrow$ & MMD $\downarrow$ & (\%) $\uparrow$ & (\%) $\uparrow$ \\
            \hline
            1.0 & 1.0 & 82.96 & \textbf{82.84} & \textbf{75.32} & 0.4914 & 0.00904 & 3.92 & 0.76 \\
            0.7 & 1.0 & 82.72 & 81.92 & 65.88 & 0.4748 & 0.00916 & 12.04 & 1.24 \\
            0.7 & 0.9 & \textbf{83.72} & 82.56 & 63.88 & 0.4682 & \textbf{0.00572} & \textbf{14.88} & \textbf{1.36} \\
            1.0 & 0.9 & 82.52 & 81.64 & 64.80 & \textbf{0.4651} & 0.00886 & 12.56 & 0.96 \\
            \bottomrule
        \end{tabular}
    }
    
    \label{tab:results-sampling}
\end{table}

\subsection{Controllability} 

Autoregressive generation can be modified by two sampling parameters: temperature $\tau$ and nucleus sampling $P$. We test different values of these parameters to see their effect on the generated samples and MSUN rate in \cref{tab:results-sampling}. We find that the sampling parameters $\tau=0.7$ and $P=0.9$ produce the highest MSUN rate. A lower temperature makes the model predict less diverse structures as it decreases the chance the model will pick something out of distribution. The nucleus sampling helps remove some of the noise from the model's output, also helping it produce more consistent results. We see this reflected in the lower unique and novel structures and the improved distance metrics.

\begin{table}[t]
    \centering
    \caption{Comparison of different policy methods for rejection sampling. Time is computed by generating 1000 structures and dividing the total time to run the program by the number of structures generated. SMACT policy method performs the best with the highest MSUN, but partial policy generation has a lower distance to true data. It is likely that the partial policy was overfitting and is making the model generate structures that are in the training data.}
    \resizebox{\textwidth}{!}{
        \begin{tabular}{lcccccccccc}
             \toprule
             & Valid & Unique & Novel & & & Metastable & MSUN & Time \\
             Policy  & (\%) $\uparrow$ & (\%) $\uparrow$ & (\%) $\uparrow$ & JSDistance $\downarrow$ & MMD $\downarrow$ & (\%) $\uparrow$ & (\%) $\uparrow$ & (s) \\
            \hline
             None & 83.72 & 82.56 & 63.88 & 0.4682 & \textbf{0.00572} & \textbf{14.88} & 1.36 & \textbf{0.182} \\
             Partial & \textbf{85.64} & \textbf{83.92} & 65.68 & \textbf{0.4785} & 0.00971 & 14.84 & 0.92 & 0.314 \\
             Full & 84.16 & 59.32 & 63.52 & 0.5309 & 0.03264 & 13.96 & 1.12 & 1.395 \\
             SMACT & 82.28 & 77.44 & \textbf{68.80} & 0.4818 & 0.01482 & 9.72 & \textbf{1.92} & 0.220 \\
             \bottomrule
        \end{tabular}
    }
    \label{tab:policy}
\end{table}

\subsection{Policy Guided Generation}

We test our three different policy setups to filter the partially generated structures. Since our method is autoregressive, rejecting partially generated structures makes more sense than in a denoising setup as we are confident that the partially generated structure will appear in the final result. \Cref{tab:policy} shows that SMACT policy rejection works best, while the other two methods actually reduce MSUN. This is likely because the methods are prematurely cutting off potentially stable crystals because they fall out of distribution compared to MP-20, and implies that maybe a more sophisticated policy model trained on a more diverse dataset is necessary. SMACT works best because it explicitly enforces physical laws on the chosen atoms to ensure that they are properly charge-balanced.

Part of the benefit of \system{} is the stepwise development of crystals. Spending a maximum of 20 autoregressive steps is much less than the 250 or more steps required for flow matching or diffusion models. Additionally, our model generates much cleaner intermediaries faster than LLMs. This means that rejection of samples happens faster and more reliably. Moreover, the addition of physics-informed policies shows that adding physics-informed systems like machine learning interatomic potentials may improve the generation performance of other models.

\subsection{\textit{De Novo} Generation}

\begin{table*}[t]
    \centering
    \caption{The following models were all trained on MP-20 and evaluated using LeMat-GenBench \cite{betala2026lematgenbenchunifiedevaluationframework}. The parameter counts were taken from \citet{chang2025sgequidiff}. Inference time is the time to predict one sample averaged over one thousand samples on a single NVIDIA RTX 2080 Ti GPU. FlowLLM was run on two NVIDIA L40S GPUs and CrysLLMGen was run on a single L40S, as required due to the size of the LLM, but both of their diffusion or flow matching steps were run on an RTX 2080 Ti. For system type, AR means autoregressive, FM means flow matching, and DF means diffusion. \system{} is the fastest autoregressive system in terms of seconds per metastable, unique, and novel crystal. }
    \resizebox{\textwidth}{!}{
        \begin{tabular}{lccccccccccccccc}
        \toprule
            & System & Inference & & & Parameter & Valid & Unique & Novel & & & Metastable & MSUN & Inference & Time/MSUN \\
            Model & Type & Steps & $\tau$ & $P$ & Count & (\%) $\uparrow$ & (\%) $\uparrow$ & (\%) $\uparrow$ & JSDistance $\downarrow$ & MMD $\downarrow$ & (\%) $\uparrow$ & (\%) $\uparrow$ & Time (s) $\downarrow$ & (s/crystal) $\downarrow$ \\
            \hline
            CDVAE \cite{xie2021cdvae} & DF & 100 & - & - & \textbf{4.9M} & \textbf{96.92} & \textbf{96.88} & \textbf{93.16} & 0.5130 & 0.0170 & 2.92 & 2.24 & 2.049 & 91.47 \\
            DiffCSP \cite{jiao2023diffcsp} & DF & 1000 & - & - & 12.3M & 96.12 & 95.32 & 68.76 & 0.4575 & 0.0045 & 27.04 & \underline{7.72} & 0.664 & \underline{8.60} \\
            DiffCSP++ \cite{jiao2024diffcsppp} & DF & 1000 & - & - & 12.3M & 93.12 & 92.84 & 68.36 & 0.2251 & 0.0053 & 15.48 & 2.76 & 2.143 & 77.64 \\
            SymmCD \cite{levy2024symmcd} & DF & 1000 & - & - & 60.4M & 61.48 & 60.84 & 38.76 & \underline{0.2192} & \textbf{0.0029} & 16.56 & 2.72 & 0.543 & 19.96 \\
            FlowMM \cite{miller2024flowmm} & FM & 250 & - & - & 12.3M & 90.60 & 90.52 & 70.44 & 0.4294 & 0.0060 & 15.12 & 2.36 & \textbf{0.098} & \textbf{4.15} \\
            MatterGen \cite{MatterGen2025} & DF & 1000 & - & - & 44.6M & 96.12 & \underline{95.72} & \underline{71.24} & 0.4356 & 0.0057 & \underline{31.60} & \textbf{14.72} & 11.243 & 76.38 \\
            CrysLLMGen \cite{khastagir2025CrysLLMGen} & AR+DF & 900 & 1.0 & 0.7 & 7B & \underline{96.16} & 83.88 & 35.00 & 0.4922 & 0.0055 & \textbf{35.24} & 1.24 & 8.421 & 679.11 \\
            FlowLLM \cite{sriram2024flowllm} & AR+FM & 250 & 0.7 & 0.9 & 70B & 77.80 & 74.80 & 45.20 & 0.5316 & \underline{0.0033} & 2.20 & 0.04 & 4.009 & 10022.50 \\
            SGEquiDiff \cite{chang2025sgequidiff} & AR+DF & 1000 & 1.0 & 1.0 & \underline{5.5M} & 92.64 & 91.36 & 56.16 & \textbf{0.2119} & 0.0058 & 26.12 & 3.40 & 0.657 & 19.32 \\
            \system{} & AR+FM & 250 & 0.7 & 0.9 & 24.5M & 82.28 & 77.44 & 68.80 & 0.4818 & 0.01482 & 9.72 & 1.92 & \underline{0.220} & 11.45 \\
        \bottomrule
        \end{tabular}
    }
    \label{tab:results-benchmarks}
\end{table*}

We analyze the performance across a variety of models trained on the MP-20 dataset. We focus on finding the model that performs best under time/MSUN. This metric matters most in industry, where speed means more possible materials to be tested in a given time span. While some models may make more MSUN structures, the extra time it takes means you could produce more metastable, unique, and novel structures by running other methods for longer. \Cref{tab:results-benchmarks} shows that our model is the third best under time/MSUN and the best out of all autoregressive models. 

\subsection{Conditional Generation}

\begin{table}[t]
    \centering
    \caption{Performance of each autoregressive system on the slab dataset. Cleavage energy and work function are the average difference between the true and the predicted cleavage energy and work function by FIRE-GNN \cite{Hsu2025FireGNN}. Time was computed by taking the total time to generate one slab per cleavage energy and work function combo in the test dataset and then dividing by the total slabs generated while being run on an NVIDIA L40s. Failure rate is the number of unparseable CIFs generated by each model. Time/Slab is the number of seconds to generate a successful slab.}
    \resizebox{\textwidth}{!}{
        \begin{tabular}{lcccccccccccc}
        \toprule
            & Parameter & & & Failure $\downarrow$ & \multicolumn{2}{c}{Cleavage Energy (eV/\AA$^2$) $\downarrow$}  & \multicolumn{2}{c}{Work Function Top (eV) $\downarrow$} & \multicolumn{2}{c}{Work Function Bottom (eV) $\downarrow$} & Time $\downarrow$ & Time/Slab $\downarrow$ \\
            Model & Count & $\tau$ & $P$  & (\%) & MAE & RMSE & MAE & RMSE & MAE & RMSE & (s) & (s/slab)\\
            \hline
            CrystalLLM  \cite{gruver2023CrystalTextLLM} & 13B & 0.9 & 0.95 & 20.65 & 170.1244 & 9354.3447 & 348.1756 & 18897.0117 & 486.4960 & 26589.1836 & 2.66 & 3.35 \\ 
            CrystaLLM-$\pi$ \cite{bone2025discoveryrecoverycrystallinematerials} & 43M & 0.9 & 0.95 & 88.60 & 0.8954 & 15.9552 & 18.9873 & 340.4899 & 8.8132 & 121.6166 & \textbf{0.78} & 6.84 \\ 
            \system{} & 24M & 0.7 & 0.9 & \textbf{0.00} & \textbf{0.1884} & \textbf{3.6829} & \textbf{3.2326} & \textbf{75.0585} & \textbf{2.3531} & \textbf{48.4465} & 2.00 & \textbf{2.00} \\
            \bottomrule
        \end{tabular}   
    }
    
    \label{tab:results-slabs}
\end{table}

We validate \system{} on a dataset of crystal slabs annotated with cleavage energy and work function \cite{schindlerWorkFunctionCleavage2024}. This dataset poses a unique challenge: each bulk crystal can yield up to 13 distinct slabs, expanding $\approx$3,000 bulk structures into $\approx$33,000 slabs with up to 90 atoms each. Moreover, surface properties like the work function and structures are absent from bulk benchmarks, making this a more realistic generation domain.

We baseline \system{} against various LLM techniques, which enable conditioned generation. For each set of cleavage energy and work functions found in the testing dataset, we generate a candidate slab per model and evaluate the performance of the generated slab based on how close its actual cleavage energy and work function are to the desired cleavage energy and work function. To compute this, we use a machine learning method called FIRE-GNN \cite{Hsu2025FireGNN}.

Table \ref{tab:results-slabs} showcases the performance of our model and baselines. It shows that \system{} performs the best at producing slabs with properties close to the true value. Additionally, we are the fastest among all the models at generating successful slabs.
 
\section{Conclusion}

We propose \system{}, an efficient and flexible generative method for guided generation of novel materials, which efficiently incorporates permutation invariance into the generative process by predicting over a space of probabilities on atom types. In standard benchmarks on the MP-20 dataset, \system{} demonstrates comparable performance to existing state-of-the-art baselines, while retaining a clear advantage in computational efficiency necessary for high-throughput screening tasks. Furthermore, we demonstrate how \system{} can be steered by policy during the autoregressive atom generation stage, allowing finer control over generated materials. We additionally demonstrate how \system{} is flexible, allowing for extensions beyond bulk crystal generation, by performing a novel conditional generation task on crystal slabs.

\paragraph{Limitations and Future Work}

One of the limitations of our method is that we use a separate model for each of the three steps in our generation process which does increase complexity, but it allows for more control over generation, including selecting the choice of periodic boundaries and enabling policy-guided generation. Future work may incorporate additional policy and physics-informed techniques into the generation pipeline.


\bibliography{neurips_2026}

\appendix

\section{Training and Sampling \system{}}

\begin{algorithm}[h]
    \begin{algorithmic}
        \State Input $\eta$
        \State Initialize $\theta$
        \While{$\theta$ is not converged}
            \For{$C = \{L,A,X\}$ {\bfseries in} $D_\text{train}$}
                \For{Every subsets $a, \Bar{a}$ {\bfseries in} $A$}
                    \If{$\Bar{a} == \emptyset$}
                        \State $\Bar{a} \leftarrow END$
                    \EndIf
                    \State $p_{\Bar{a}} = \frac{1}{|\mathbf{a}|} \sum_{j=1}^{|\mathbf{a}|} \mathbf{e}_{a_j}$
                    \State $p_a = \AtomGenerator_\theta(start, L, a)$
                    \State $\ell = D_{KL}(p_a || p_{\Bar{a}})$
                    \State $\theta = \theta - \eta \nabla_\theta \ell$
                \EndFor
            \EndFor
        \EndWhile
    \end{algorithmic}
    \caption{Training $\AtomGenerator$}
    \label{alg:automatgen-train}
\end{algorithm}

\begin{algorithm}[h]
    \begin{algorithmic}
        \State Input $numSteps, maxAtoms$
        \State $L \sim \LatticeGenerator(\mathbb{R}^{3 \times 3})$
        \State Initialize $A = [], end = False$
        \While{not $end$ and $A$.length() < $maxAtoms$}
            \State $a = \AtomGenerator(start, L, A)$
            \If{$a$ is $END$}
                \State $end = True$
            \Else
                \State $A \leftarrow a$
            \EndIf
        \EndWhile
        \State $N = len(atoms)$
        \State Initialize $X ~ \uniform([0,1)^{N \times 3})$
        \For{$t=1$ {\bfseries to} $0$ {\bfseries in} $numSteps$}
            \State $v = \PositionGenerator(L, A, X, t)$
            \State $X = X - v/numSteps$
        \EndFor
        \State \textbf{return} $L, A, X$
    \end{algorithmic}
    \caption{Sampling from \system{}}
    \label{alg:automatgen-sample}
\end{algorithm}

\section{Proof of Permutation-Invariant Training}

\begin{proposition}
    \label{prop:train-inv}
    Given a crystal $C = (L,A,X)$, for any list of input atoms $a_{in} \subset A$ and output atoms $a_{out} = A \setminus a_{in}$, no matter the permutations $\pi_1, \pi_2$ applied to $a_{in}$ and $a_{out}$, $\AtomGenerator$ would receive the same training update. 
\end{proposition}
\begin{proof}
    Let $p(a_{out})$ be the categorical distribution of $a_{out}$. Since $p$ is a distribution, it is invariant to any permutation of atoms $p(a_{out}) = p(\pi_2(a_{out}))$ as it only depends on the set which $\pi_2$ doesn't affect. The model $\AtomGenerator$ is permutation invariant by construction as GNNs ignore node ordering due to aggregating over all nodes, so $\AtomGenerator(L, start, a_{in}) = \AtomGenerator(L, start, \pi_1(a_{in}))$. Therefore each training update will be permutation invariant as the loss of each step will be permutation invariant $D_{KL}(\AtomGenerator(L, start, a_{in}), p(a_{out})) = D_{KL}(\AtomGenerator(L, start, \pi_1(a_{in}))), p(\pi_2(a_{out})))$. 
\end{proof}

\section{Example Crystals}
\label{app:example-crystals}

\begin{figure}[h]
    \centering
    \begin{subfigure}{0.3\textwidth}
        \includegraphics[width=\textwidth]{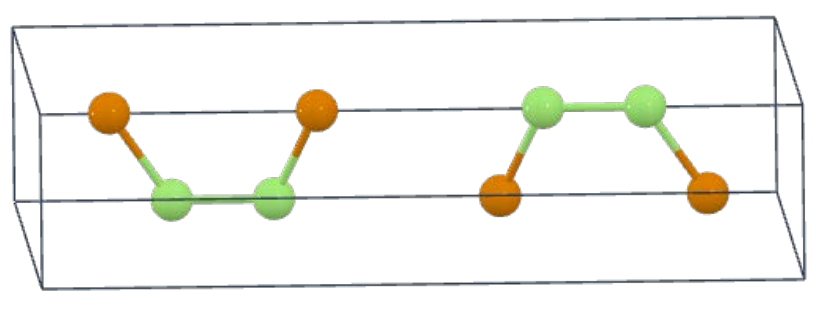}
        \caption{Example of a real Ga$_4$Te$_4$ crystal from MP-20}
    \end{subfigure}
    \begin{subfigure}{0.3\textwidth}
        \centering
        \includegraphics[width=\textwidth]{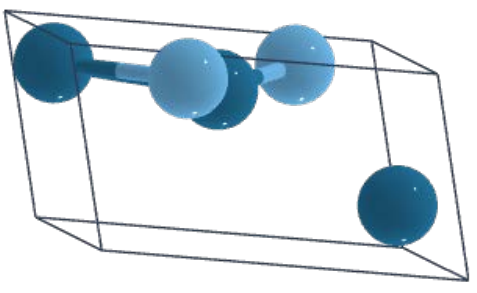}
        \caption{Example of a generated Re$_3$Ti$_2$ crystal before relaxation.}
    \end{subfigure}
    \begin{subfigure}{0.3\textwidth}
        \includegraphics[width=\textwidth]{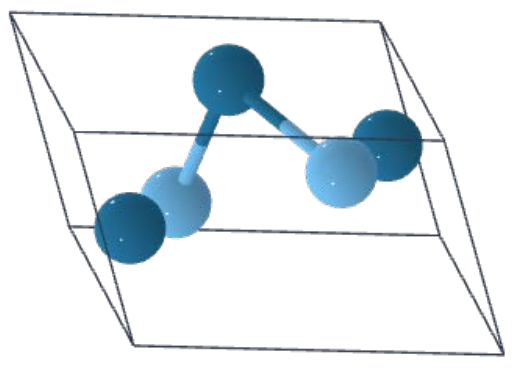}
        \caption{Example of a generated Re$_3$Ti$_2$ crystal after relaxation.}
    \end{subfigure}
    \caption{Examples of real and generated crystals visualized by the Material Project's crystal toolkit \cite{Jain2013materialsproject}.}
    \label{fig:generated-examples}
\end{figure}



\end{document}